\pdfoutput=1
\documentclass{ecai}
\usepackage{times}
\usepackage{amsmath}
\usepackage{amsthm}
\usepackage{amssymb}
\usepackage{booktabs}
\usepackage{algorithmic}
\usepackage[ruled, noend]{algorithm2e}
\usepackage{multirow}
\usepackage[caption=false]{subfig}
\usepackage{amsfonts}
\usepackage{adjustbox}
\newtheorem{thm}{Theorem}

\newtheorem{proposition}[theorem]{Proposition}

\usepackage{color}

\newcommand{\ignore}[1]{}
\usepackage{graphicx}
\usepackage{latexsym}

\ecaisubmission

\begin{document}

\title{Evolutionary Bi-objective Optimization for the Dynamic Chance-Constrained Knapsack Problem Based on Tail Bound Objectives}

\author{Hirad Assimi \institute{The University of Adelaide, Australia, email: hirad.assimi@adelaide.edu.au} \and Oscar Harper \institute{The University of Adelaide,
Australia, email: oscar.harper@student.adelaide.edu.au} \and Yue Xie \institute{The University of Adelaide, Australia, email: yue.xie@adelaide.edu.au} \and Aneta Neumann \institute{The University of Adelaide, Australia, email: aneta.neumann@adelaide.edu.au} \and Frank Neumann \institute{The  University of Adelaide, Australia, email: frank.neumann@adelaide.edu.au}
}
\maketitle

\begin{abstract}
Real-world combinatorial optimization problems are often stochastic and dynamic. Therefore, it is essential to make optimal and reliable decisions with a holistic approach. In this paper, we consider the dynamic chance-constrained knapsack problem where the weight of each item is stochastic, the capacity constraint changes dynamically over time, and the objective is to maximize the total profit subject to the probability that total weight exceeds the capacity. We make use of prominent tail inequalities such as Chebyshev's inequality, and Chernoff bound to approximate the probabilistic constraint. Our key contribution is to introduce an additional objective which estimates the minimal capacity bound for a given stochastic solution that still meets the chance constraint. This objective helps to cater for dynamic changes to the stochastic problem. We apply single- and multi-objective evolutionary algorithms to the problem and show how bi-objective optimization can help to deal with dynamic chance-constrained problems.
\end{abstract}

\section{\uppercase{Introduction}}

Many real-world combinatorial optimization problems involve stochastic as well as dynamic components which are mostly treated in isolation. However, in order to solve complex real-world problems, it is essential to treat stochastic and dynamic aspects in a holistic approach and understand their interactions. 
 
Dynamic components in an optimization problem may change the objective function, constraints or decision variables over time. The challenge to tackle a dynamic optimization problem (DOP) is to track the moving optima when changes occur \cite{Nguyen2012}.
 
Moreover, uncertainty is pervasive in a real-world optimization problem. The source of uncertainty may involve the nature of data, measurement errors or lack of knowledge. Ignoring uncertainties in solving a problem may lead to obtaining suboptimal or infeasible solutions in practice \cite{Li2015}. 
 
Chance-constrained programming (CCP) is a powerful tool to model uncertainty in optimization problems. It transforms an inequality constraint into a probabilistic constraint to ensure that the probability of constraint violation is smaller than a limit predefined by the decision-maker \cite{Charnes1959}. CCP has been applied successfully in different domains such as process control, scheduling and supply management where safety requirements are concerned \cite{Farina2016}. 
 
Evolutionary algorithms (EAs) have been applied to many combinatorial optimization problems and demonstrate a high capability in solving hard problems, including a wide range of real-world applications \cite{Michalewicz1994, Chiong2012}. Multi-objective EAs deal with several (conflicting) objectives and provide a set of solutions which are non-dominated to each other with respect to the given objective functions \cite{deb2001a,DBLP:conf/nips/QianYZ15,DBLP:conf/ijcai/QianSYT17}. 
 
In addition to solving problems with conflicting objectives, several studies have indicated that transforming a single-objective optimization problem to a multi-objective optimization problem may lead to obtaining better solutions. This transformation leads to obtaining a set of non-dominated solutions instead of a single solution. Therefore, each individual in the Pareto front contains helpful information which can improve the performance of the algorithm in exploring the search space \cite{Neumann2006,DBLP:conf/nips/QianYZ15,DBLP:conf/ijcai/QianSYT17}.

\subsection{Related work}
EAs are a natural way to deal with DOPs because they are inspired by nature which is an ever-changing environment \cite{Nguyen2012}. The behaviour of EAs has been analyzed on a knapsack problem with dynamically changing constraint \cite{Roostapour2018}. They carried out bi-objective optimization with respect to the profit and dynamic capacity constraint. They proposed an algorithm to track the moving optimum and showed that handling the constraint by a bi-objective approach can be beneficial in obtaining better solutions when the changes occur less frequently. Their studies have been extended to analyze the EA behaviour dependent on the submodularity ratio of a broad class of problems \cite{Roostapour2019}.
 
Chance-constrained knapsack problem (CCKP) is a variant of the classical NP-hard deterministic knapsack problem where the weights and profits can be stochastic \cite{Kellerer2004}. Approximation algorithm in combination with a robust approach, has been applied to CCKP to find feasible solutions for a simplified knapsack problem \cite{DBLP:journals/orl/KlopfensteinN08}. 
 
Recently, Xie et al.~\cite{Yue2019} integrated inequality tails with single and bi-objective EAs to solve CCKP. To estimate the probability of constraint violation, they used popular tail inequalities. They investigated the behaviour of Chebyshev’s inequality, and Chernoff bound for approximation in CCKP. They also carried out bi-objective optimization with respect to the profit and probability of constraint violation when the capacity is static. Doerr et al.~\cite{DoerrAAAI20} have investigated adaptations of classical greedy algorithms for the optimization of submodular functions with chance constraints of knapsack type. They have shown that the adapted greedy algorithms maintain asymptotically almost the same approximation quality as in the deterministic setting when considering uniform distributions with the same dispersion for the knapsack weights.

\subsection{Our Contribution}
In this paper, we consider the dynamic chance-constrained knapsack problem (DCCKP) with dynamically changing constraint. We also assume that each item in the knapsack problem has an uncertain weight while the profits are deterministic.
For the dynamic component, we follow the settings defined in \cite{Roostapour2018}: the knapsack capacity changes over time every $\tau$ iterations with a predefined magnitude. Moreover, for the stochastic component, we follow the approach and settings proposed in \cite{Yue2019} which employs inequality tails to estimate the violation in probabilistic constraint.
 
Therefore, the goal in this study is to re-compute a solution of maximal profit after a dynamic change occurs to the capacity constraint, while the total uncertain weight can exceed the capacity with a small probability.
To benefit from the bi-objective optimization, we cannot directly apply the second objective used in previous studies because they only considered either dynamic or stochastic aspect of the optimization problem in isolation for the second objective function.
Therefore, we introduce an objective function which deals with uncertainties and caters for dynamic aspects of the problems. This objective evaluates the smallest knapsack capacity bound for which a solution would not violate the chance constraint. This objective also can keep a set of non-dominated solutions to be used for tracking the moving optimum. This objective makes use of tail inequalities such as Chebyshev's inequality and Chernoff bounds to approximate the probabilistic constraint violation.

To solve DCCKPs, we apply a single objective EA, a modified version of GSEMO~\cite{Giel2003} and NSGA-II~\cite{NSGA2}, where the last two compute a trade-offs with respect to the profit and the newly introduced objective for dealing with chance constraints. Our experimental results show that the bi-objective EAs perform better than the single objective approaches. Introducing the additional objective function to the problem helps the bi-objective optimization algorithm to deal with the constraint changes as it obtains the non dominated solutions with respect to the objective functions.
 
The rest of this article is organized as follows. In the next section, we define the DCCKP and introduce the two tail inequalities for quantifying uncertainties. Afterwards, we introduce the objective function for dealing with DCCKP and develop the bi-objective model. Next, we report on the behaviour of single objective and bi-objective baseline EAs in solving DCCKP. We show that bi-objective optimization with the introduced second objective can obtain better solutions on a wide range of instances of the DCCKP. Finally, we finish with some concluding remarks.
 
\section{\uppercase{Dynamic Chance-Constrained Knapsack Problem}}\label{sec.problem}
In this section, we introduce the problem and provide Chebyshev and Chernoff tail inequalities to estimate the probability of chance constraint violation in the problem.

\subsection{Problem Formulation}
The classical knapsack problem can be defined as follows. Given $n$ items where each item $i$, $1 \leq i \leq n$ has a profit $p_i$ and a weight $w_i$ and a knapsack capacity $C$. The goal is to find a selection of items of maximum profit whose weight does not exceed the capacity bound. 
A candidate solution is an element $x \in \{0,1\}^n$ where item $i$ is chosen iff $x_i=1$.
In this paper, we consider the stochastic and dynamic setting for the knapsack problem where each weight is chosen independently according to a given probability distribution. Furthermore, the capacity bound $C$ changes dynamically over time.

The search space is $\{0,1\}^n$ and we denote by $P(x) = \sum_{i=1}^n p_i x_i$ the profit and by $W(x) = \sum_{i=1}^n w_ix_i$ the weight of a solution $x$.
We investigate the chance-constrained knapsack problem where the goal is to maximize $P(x)$ under the condition that the probability that the weight of the solution is at least as high as the capacity is at most $\alpha$.  Formally, we define this constraint as
$$\text{Pr}\left[W(x) \geq C\right] \leq \alpha$$
where $\alpha$ is a parameter that upper bounds the probability of exceeding the knapsack capacity $(0<\alpha<1)$. 

Furthermore, the knapsack capacity in our problem is dynamic and changes over time every $\tau$ iterations. We call $\tau$ the frequency of changes which denotes after how many iterations a change occurs in the knapsack capacity with the magnitude of changes $r$ according to some probability distributions. 

\subsection{Tail Bounds}  \label{sec.single}
Chebyshev's inequality tail can determine a bound for a cumulative distribution function of a random design variable. Chebyshev's inequality requires to know the standard deviation of the design variables and gives a tighter bound in comparison to the weaker tails such as Markov's inequality. Therefore, it can be applied to any distribution if the expected weight and standard deviation of the involved random variables are known. The standard Chebyshev inequality is two-sided and provides tails for upper and lower bounds \cite{Casella2002}. As we are only interested in the probability of exceeding the weight bound, we use a one-sided Chebyshev inequality which is also known as Cantelli's inequality \cite{DasGupta2008}. For brevity, we refer to the one-sided Chebyshev as Chebyshev's inequality in this paper. 
\begin{thm}[One-sided Chebyshev inequality]
Let $X$ be an independent random variable, and let $\mathbb{E}(X)$ denote the expected weight of $X$. Further, let $\sigma^2_{X}$ be the variance of $X$. Then for any $\lambda \in \mathbb{R}^+$, we have
\begin{equation*}
\text{Pr}\left[\left(X-\mathbb{E}(X)\right) \geq\lambda\right]\leq\frac{\sigma^{2}_{X}}{\sigma^{2}_{X}+\lambda^{2}}.\\
\end{equation*}
\end{thm}
Compared to Chebyshev's inequality, Chernoff bound provides a sharper tail with an exponential decay behavior. In order to use Chernoff bound, it is essential that the random variable is a summation of independent random variables. 
Chernoff bound seeks a positive real number $t$ in order to find the probability where the sum of independent random variables exceeds a particular threshold \cite{Motwani1995}. Therefore, Chernoff bound for an independent variable $X$ can be given as follows based on theorem 2.3 in \cite{McDiarmid1998}.
\begin{thm}
Let $X= \sum_{i=1}^n X_i$ be the sum of independent random variables $X_i \in [0, 1]$ chosen uniformly at random, and let $\mathbb{E}(X)$ be the expected weight of $X$. For any $t>0$, we have	
\begin{equation*}
	\text{Pr}[X \geq (1+t)\mathbb{E}(X)] \leq \exp\left(-\frac{t^2}{2+\frac{2}{3}t}\mathbb{E}(X)\right).
\end{equation*}
\end{thm}

\section{\uppercase{Bi-objective Optimization for DCCKP}} \label{sec.multi}
In this section, we introduce a new objective-function to transform the single-objective optimization problem into a bi-objective optimization problem. We also describe (1+1)-EA and POSDC as baseline single and bi-objective EAs.

\subsection{Bi-objective Model}  \label{subsec.fitness}
We redefine DCCKP by introducing a new second objective function to transform it into a bi-objective optimization problem. Therefore, we introduce $(C^*)$ as the stochastic bound as our second objective function. This objective function evaluates the smallest knapsack capacity for a given solution such that it satisfies the predefined limit on the chance constraint. Therefore, the fitness $f(x)$ of a solution $x$ is given as 
\begin{equation*}
    f(x)=(P(x),C^*(x))
\end{equation*}
where
\begin{equation*}
    C^*(x)=\min\left\{C\;|\;\text{s.t.}\; \text{Pr}[W(x) \geq C] \leq \alpha \right\}
\end{equation*}
is the smallest weight bound $C$ such that the probability that the weight $W(x)$ of $x$ is at least $C$ is at most $\alpha$. Using this objective allows to cater for dynamic changes of the weight bound of our problem. In bi-objective optimization of DCCKP, the goal is to maximize $P(x)$ and minimize $C^*(x)$. Hence, we have
$$f(x') \succeq_{} f(x) \: \text{iff} \:P(x') \geq P(x) \land C^*(x') \leq C^*(x)$$
for the dominance relation of bi-objective optimization for two solutions namely $x$ and $x'$.
Evaluating the chance constraint is computationally difficult \cite{Ben2009}. It has been shown that even if random variables are from a Bernoulli distribution, calculating the probability of violating the constraint exactly is \#P-complete, see Theorem 2.1 in \cite{Kleinberg2000}. Because it is difficult to compute $C^*$ exactly, we make use of the tail inequalities to calculate the second objective function. For Chebyshev's inequality, the stochastic bound is given as follows.
\begin{proposition}[Chebyshev Constraint Bound Calculation] \label{prop.cheby}
Let $\mathbb{E}(W(x))$ be the expected weight, $\sigma^2_{W(x)}$ be the variance of the weight of solution $x$ and $\alpha$ be the probability bound of the chance constraint.
Then setting $C_1^*(x)=\mathbb{E}(W(x))+\frac{\sigma_{W(x)}\sqrt{\alpha(1-\alpha)}}{\alpha}$
implies $\text{\upshape Pr}[W(x) \geq C_1^*(x)] \leq \alpha.$
\end{proposition}
\begin{proof} 
Using Chebyshev's inequality, we have
\begin{align*}
&\text{Pr}\left[W(x) \geq \mathbb{E}(W(x)) + \lambda \right] \leq \frac{\sigma^{2}_{W(x)}}{\sigma^{2}_{W(x)}+\lambda^{2}}.
\end{align*}
We set $C_1^*(x)= \mathbb{E}(W(x)) + \lambda$ which implies
\begin{align*}
\lambda = \frac{\sigma_{W(x)}\sqrt{\alpha(1-\alpha)}}{\alpha}.
\end{align*}
Hence, we have 
\begin{align*}
&\text{Pr}[W(x) \geq C_1^*(x)]\\
&=\text{Pr}\left[W(x) \geq \mathbb{E}(W(x) + \frac{\sigma_{W(x)}\sqrt{\alpha(1-\alpha)}}{\alpha} \right]\\
&\leq \frac{\sigma^{2}_{W(x)}}{\sigma^{2}_{W(x)}+\left(\frac{\sigma_{W(x)}\sqrt{\alpha(1-\alpha)}}{\alpha}\right)^{2}}\\
&=\alpha
\end{align*}
which completes the proof.
\end{proof}
We consider $w_i \in \mathcal{U}[\mathbb{E}(w_i)-\delta, \mathbb{E}(w_i)+\delta]$ and $\sum_{i=1}^n x_i$ denotes the total number of chosen items in a solution, then the stochastic bound based on Chebyshev's inequality for uniform distribution is given as follows
$$\sigma_{W(x)}=\delta\sqrt{\frac{\sum_{i=1}^n x_i}{3}}.$$
We substitute $\lambda$ as
$$\lambda = \frac{\delta\sqrt{3\alpha(1-\alpha)\sum_{i=1}^n x_i}}{3\alpha}$$
Therefore, we have
$$C_1^*(x)=\mathbb{E}(W(x))+\frac{\delta\sqrt{3\alpha(1-\alpha)\sum_{i=1}^n x_i}}{3\alpha}.$$
Moreover, to derive the second objective function by making use of the Chernoff bound, we have:
\begin{proposition}[Chernoff Constraint Bound Calculation]
Let $w_i \in \mathcal{U}[\mathbb{E}(w_i)-\delta, \mathbb{E}(w_i)+\delta]$ be independent weights chosen uniformly at random.
Let $\mathbb{E}(W(x))$ be the expected weight of $x$ and $\alpha$ be the probability bound of the chance constraint.
Then setting
\begin{align*}
&C_2^{*}(x)=\mathbb{E}(W(x))\\
&-0.66\delta\left(\ln(\alpha)-\sqrt{\ln^2(\alpha)-9\ln(\alpha)\sum_{i=1}^n x_i}\right)
\end{align*}
implies $\text{Pr}[W(x) \geq C_2^*(x)) \leq \alpha$.
\end{proposition}
\begin{proof} 
We consider $w_i \in \mathcal{U}[\mathbb{E}(w_i)-\delta, \mathbb{E}(w_i)+\delta]$. Then, to satisfy Chernoff bound summation requirement, we normalize each random weight into [0, 1] which $y_i$ denotes the normalized weight of item.
$$y_i=\frac{w_i-(\mathbb{E}(w_i)-\delta)}{2\delta} \in [0,1]$$
$$ Y(x)=\sum_{i=1}^ny_i=\sum_{i=1}^n\frac{w_i-(\mathbb{E}(w_i)-\delta)}{2\delta}x_i.$$
Since $y_i$ is symmetric, then the total expected weight of $Y$ is
$\mathbb{E}(Y)=\frac{1}{2}\sum_{i=1}^n x_i.$ Then, the total weight of a solution is given as
$$W(x)=\sum_{i=1}^n w_ix_i=2\delta Y(x)+\mathbb{E}(W(x))-\delta\sum_{i=1}^n x_i.$$
We set
$$C^*_2=\mathbb{E}(W(x))+b$$
where
$$b=-0.66\delta\left(\ln(\alpha)-\sqrt{\ln^2(\alpha)-9\ln(\alpha)\sum_{i=1}^n x_i}\right).$$

Hence, the probability of violating the chance constraint. for a solution is given as
\begin{align*}
& \text{Pr}\left[W(x) \geq C^*_2(x) \right]\\
= &\text{Pr}\left[2\delta Y(x)+\mathbb{E}(W(x))-\delta\sum_{i=1}^n x_i \geq \mathbb{E}(W(x))+b\right]\\
 = & \text{Pr}\left[Y(x) \geq \frac{1}{2}\sum_{i=1}^n x_i + \frac{b}{2\delta}\right]\\
= & \text{Pr}\left[Y(x) \geq \mathbb{E}\left(Y(x)\right) + \frac{b}{2\delta}\right]\\
= & \text{Pr}\left[Y(x) \geq (1+t)\mathbb{E}\left(Y(x)\right)\right]
\end{align*}
where
$$t= \frac{b}{2\delta \mathbb{E}\left(Y(x)\right)}.$$  
Using Chernoff bounds, we have
\begin{eqnarray*}
& & \text{Pr}\left[\left(Y(x) \geq (1+t)\mathbb{E}\left(Y(x)\right) \right)\right]\\
&\leq & \exp\left(-\frac{t^2}{2+\frac{2}{3}t}\mathbb{E}(Y(x))\right)
\end{eqnarray*}
We have
$$t= \frac{-0.66\left(\ln(\alpha)-\sqrt{\ln^2(\alpha)-9\ln(\alpha)\sum_{i=1}^n x_i}\right)}{\sum_{i=1}^n x_i}.$$
We use
$$\hat{t}=\frac{-0.66\ln(\alpha)}{\sum_{i=1}^n x_i}\leq \frac{b}{2\delta \mathbb{E}(Y(x))}=t $$
instead of $t$ which results in
\begin{align*}
& & \text{Pr}\left[Y(x) \geq (1+t)\mathbb{E}\left(Y(x)\right)\right] \\
&\leq &\text{Pr}\left[Y(x) \geq (1+\hat{t})\mathbb{E}\left(Y(x)\right)\right]\\
    &\leq & \exp \left ( -\frac{\frac{(-0.66)^2\ln^2\alpha}{\left(2\mathbb{E}\left(Y(x)\right)\right)^2}}{2+\frac{2}{3}(\frac{-0.66\ln \alpha}{2\mathbb{E}\left(Y(x)\right)})} \cdot \mathbb{E}\left(Y(x)\right)  \right)\\
&= & \exp\left(\frac{(-0.66)^2\cdot\ln\alpha\cdot \ln\alpha}{-8\mathbb{E}\left(Y(x)\right)+\frac{4}{3}(0.66\ln \alpha)}\right) \\    
& \leq \alpha.   
    \\
\end{align*}
The last inequality holds as
\begin{align*}
& -4\sum_{i=1}^n x_i+\frac{4}{3}(0.66\ln \alpha) \leq (-0.66)^2\ln \alpha
\end{align*}
which completes the proof.
\end{proof}
Note that the introduced additional objectives as $C^*_1$ and $C^*_2$ calculate the smallest possible bound for which a solution meets the chance constraint according to the used tail bound (Chebyshev or Chernoff). The terms added to the expected total weight guarantee that a given solution meets the chance constraint. 
\subsection{POSDC Algorithm}
\LinesNumbered
\begin{algorithm}[t]
\caption{POSDC}\label{alg.POSDC}
\DontPrintSemicolon
Generate $x \in \{0,1\}^n$ uniformly at random \;
\uIf {$C-\eta \leq C^*(x) \leq C+\eta$}
{$S \gets x$}
\uElse{
\While {$S=\varnothing$}{
repair an offspring ($y$) by (1+1)-EA \;
$x \gets y$\;
\uIf {$C-\eta \leq C^*(y) \leq C+\eta$}
{$ S \gets x $}
}
}
\While {\upshape $\text{(not max iteration)}$}{
\uIf {\upshape $\text{change in the capacity occurs (after $\tau$ iterations)}$} 
{$x \gets$ best solution in $S$\\
Update $S^-$ and $S^+$ with respect to the shifted capacity\\
\uIf {$S=\varnothing$}
{$S \gets x$}
}
choose $x \in S$ uniformly at random\;
$y$ $\gets$ create an offspring by flipping each bit of $x$ independently with the probability of $\frac{1}{n}$\;
\uIf {$(C-\eta \leq C^*(y) < C) \land ( \nexists z \in S^- : z \succeq_{\text{POSDC}} y)$}
{$S^- \gets (S^- \cup {y}) \setminus \{z \in S^- | y \succeq_{\text{POSDC}} z\} $}
\uElseIf {$(C \leq C^*(y) \leq C+\eta) \land ( \nexists z \in S^+ : z \succeq_{\text{POSDC}} y)$}
{$S^+ \gets (S^+ \cup {y}) \setminus \{z \in S^+ | y \succeq_{\text{POSDC}} z\} $}
}
\Return \textit{best solution}
\end{algorithm}
We adapt the algorithm proposed in \cite{Roostapour2018} for our bi-objective optimization. We call the adapted algorithm, Pareto Optimization for Stochastic Dynamic Constraint (POSDC) which deals with DCCKP. POSDC (see Algorithm \ref{alg.POSDC}) is a baseline multi-objecitve EA which tracks the moving optimum by storing a population in the vicinity of the dynamic knapsack capacity. POSDC keeps a solution ($x$) if $C^*(x)$ is in $[C-\eta, C+\eta]$, where $\eta$ determines the storing range. 
Therefore, POSDC has two subpopulations which include feasible and infeasible solutions ($S=S^- \cup S^+$). Keeping an infeasible subpopulation helps POSDC to be prepared for the next change in the dynamic constraint.
\begin{equation*}
           \begin{aligned}
                       &S^- \gets \{x \in S \mid C-\eta \leq C^*(x) \leq C\}\\
                       &S^+ \gets \{x \in S \mid C < C^*(x) \leq C+\eta\}.
           \end{aligned}
\end{equation*}
POSDC generates the initial solution uniformly at random, if the generated solution is out of the storing range, then (1+1)-EA (see Algorithm~\ref{alg.oneplusone}) repairs the solution and stores it in the appropriate subpopulation. (1+1)-EA is a single-objective baseline EA which is described later.
 
POSDC uses a mutation operator to explore the search space and find trade-off solutions. POSDC maintains a set of non-dominated solutions with respect to $P(x)$ and $C^*(x)$ in its subpopulations. The best solution in POSDC at each iteration is the solution with the highest profit in $S^-$; If $S^-$ is empty, POSDC prefers the solution with the smallest $C^*$ in $S^+$.

Note that if we can compute the solutions exactly, some solutions in $S^+$ can be feasible. However, because computing $C^*$ in exact is difficult, we designate the optimum as the solution with the highest profit in $S^-$. 
\subsection{Single Objective Approach} \label{subsec.single}
We only use simple baseline algorithms to make a fair comparison between the single-objective optimization and bi-objective optimization. (1+1)-EA and GSEMO \cite{Giel2003} are their equivalent counterparts if we consider identical objective functions because they use the same mutation operator. In this study, we adapt POSDC as a variant of GSEMO to tackle both dynamic and chance-constrained components of the problem. Therefore, we show the efficiency of bi-objective optimization by comparing POSDC with (1+1)-EA (see Algorithm~\ref{alg.oneplusone}). 

(1+1)-EA generates one potential solution uniformly at random; In each iteration, an offspring $x'$ is produced by flipping each bit of $x$ with probability $1/n$ \cite{DROSTE2002}. The offspring $x'$ replaces $x$ if it is fitter with respect to the fitness of a solution which is as follows,
$$f_{(1+1)}(x)= \left(\max\{0,\alpha(x)-\alpha\}, P(x)\right)$$
where $\alpha(x)$ denotes the probability of chance constraint violation based on Chebyshev's inequality or Chernoff bound derived for CCKP with uniform distribution in \cite{Yue2019} as follows, 
\begin{equation*}
\resizebox{.95\columnwidth}{!}{
$
           \small \normalsize
           \begin{aligned}
&   \text{Pr}_{\text{Chebyshev}} \equiv \text{Pr}\left[W(x) \geq C\right] \leq\frac{\delta^{2}\sum_{i=1}^n x_i}{\delta^{2}\sum_{i=1}^n x_i + 3(C-\mathbb{E}(W(x)))^{2}} \text{,}\\                   
                       & \text{Pr}_{\text{Chernoff}} \equiv\text{\text{Pr}}[W(x) \geq C]\\
                       & \leq \exp\left(-\frac{3\left(C-\mathbb{E}\left(W(x)\right)\right)^2}{4\delta\left(3\delta\sum_{i=1}^n x_i+C-\mathbb{E}\left(W(x)\right)\right)}\right).
           \end{aligned}
$
}
\end{equation*}
The fitness function $f_{(1+1)}$ is in lexicographic order which means that first, the algorithm searches for a feasible solution according to the chance constraint and optimizes the profit afterwards. We have,
 
\begin{eqnarray*}
& &f_{(1+1)}(x') \succeq f_{(1+1)}(x) \\
&\Longleftrightarrow& (\max\{0,\alpha(x')-\alpha\} < \max\{0,\alpha(x)-\alpha\})\\
& &\vee ~~
((\max\{0,\alpha(x')-\alpha\} = \max\{0,\alpha(x)-\alpha\})\\
& & ~~\wedge (P(x') \geq P(x)))
\end{eqnarray*}
 
\LinesNumbered
\begin{algorithm}[t]
\caption{(1+1)-EA}\label{alg.oneplusone}
\DontPrintSemicolon
generate $x \in \{0,1\}^n$ uniformly at random \;
\While {termination criterion not satisfied}{
$y$ $\gets$ create an offspring by flipping each bit of $x$ independently with the probability of $\frac{1}{n}$\;
\If{$f_{(1+1)}(y) \succeq f_{(1+1)}(x)$}
{$x \gets y$}
}
\Return $x$
\end{algorithm}
\begin{table}
\caption{Corresponding weight and profit interval for knapsack problems benchmark 
}
\label{table.ttp}
\centering
\begin{tabular}{@{}lll@{}}
\toprule
type&weight $(w_{i})$ &profit\\
\midrule
Uncorrelated&[1,1000]&[1,1000]\\
Bounded strongly correlated&[1,1000]&$\mathbb{E}(w_{i})+\text{\c{c}}$\\
\bottomrule
\end{tabular}
\end{table}
When a change occurs in the dynamic constraint, the individual ($x$) may become infeasible, and its probabilistic constraint violates $ \alpha$. Therefore, (1+1)-EA mutates $x$ to find a feasible solution for the newly given constraint and optimizes the profit afterwards.

\section{\uppercase{Experimental Investigation}} \label{sec.expr}
In this section, we define the setup of our experimental investigation; we apply the bi-objective optimization with the introduced objectives and compare it with the single-objective optimization. 
\subsection{Experimental Setup} \label{subsec.setup}
For this study, we use the binary knapsack test problems introduced in \cite{TTP2014} and later developed for dynamic knapsack problem in \cite{Roostapour2018}. We consider two types of \emph{uncorrelated} and \emph{bounded strongly correlated} test problems. The latter is more difficult to solve because the profit correlates with the weight \cite{TTP2014}. Note that in our chance-constrained setting, for bounded strongly correlated instances, we consider the correlation between the expected weight and the profit. Table \ref{table.ttp} lists the corresponding weight and profit for each type of knapsack instance where \c{c} denotes a constant number \cite{TTP2014}.
 
For the Dynamic parameters of test problems, we define $r$ which determines the magnitude of changes. we consider changes according to the uniform distribution in $[-r,r]$ where $r\in\{500, 2000\}$ to consider the small and large magnitude of changes in the knapsack constraint, respectively. 

Also, the parameter $\eta$ for POSDC has been considered equal to $r$ to cover the interval of the uniform distribution entirely for storing desirable solutions \cite{Roostapour2018}.
 
Another dynamic parameter is the frequency parameter of $\tau$, which determine how many iterations there are between dynamic constraint changes. We set $\tau\in\{100, 1000\}$ to observe fast and slow changes in the constraint, respectively.
 
For stochastic parameters, we set $\alpha\in\{0.01, 0.001, 0.0001\}$ to consider loose and tight probability of chance-constraint violation probability. We also set $\delta\in\{25, 50\}$ to assign small and large uncertainty interval in the weight of items with uniform distribution. To ensure that the weights of items subject to uncertainty are positive, we also add a value of 100 to all weights to avoid negative values.
 
We use dynamic programming to find the exact optimal solution with maximal profit $P(x^*)$ for the deterministic variant of the knapsack problem. Therefore, we record $P(x^*)$ for every dynamic capacity change for each knapsack instances based on $r$ and $\tau$.

To evaluate the performance of our algorithms for DCCKP, we consider the offline error which represents the distance between the algorithm best-obtained solution in each iteration with respect to $P(x^*)$. Let $x$ be the best solution obtained by the considered algorithm in iteration $i$. The offline error for iteration $i$ is given as

\begin{equation*}
\label{eq.error}
\resizebox{.95\columnwidth}{!}{
$         \phi_i=\begin{cases}
P(x^*)-P(x) & \text{if} \: \text{Pr}\left[W(x) \geq C\right] \leq \alpha \\
\\[1pt]
\left(1+\text{Pr}\left[W(x) \geq C\right]\right)P(x^*) & \text{otherwise.} \\
\end{cases}\\
$
}
 \end{equation*}
 Note, that every solution $x$ not meeting the chance constraint receives a higher offline error than any solution meeting the chance constraint.
The total offline error 
\begin{equation*}
\Phi=\frac{\sum_{i=1}^{10^6}\phi_i}{10^6}.\\
\end{equation*}
is the summation of offline error at each iteration divided by the number of total iterations ($10^6$).

\subsection{Experimental Results} \label{sec.results}
\begin{table*} 
\caption{Statistical results of total offline error for (1+1)-EA and POSDC with small change in the dynamic constraint}
\label{table.500}
\begin{adjustbox}{max width=2\columnwidth}
\centering
\begin{tabular}{@{}cllllrrlrrlrrlrrl@{}}
\toprule
&$r$ &$\tau$ &$\delta$ &$\alpha$ &\multicolumn{3}{c}{(1+1)-EA-Chebyshev (1)} & \multicolumn{3}{c}{(1+1)-EA-Chernoff (2)} & \multicolumn{3}{c}{POSDC-Chebyshev (3)} & \multicolumn{3}{c}{POSDC-Chernoff (4)} \\
\cmidrule{6-8} \cmidrule{9-11} \cmidrule{12-14} \cmidrule{15-17}
\parbox[t]{2mm}{\multirow{16}{*}{\rotatebox[origin=c]{90}{uncorrelated}}} &     &      &       &          & Mean      & Std     & Stat    & Mean     & Std    & Stat    & Mean      & Std     & Stat    & Mean     & Std    & Stat    \\
\midrule
&500  & 100  & 25  & 0.01   & 4232.51  & 475.50  & $2^{(*)}, 3^{(-)}, 4^{(-)}$ & 4288.23 & 481.60 & $1^{(*)}, 3^{(-)}, 4^{(-)}$ & 1485.56  & 177.13  & $1^{(+)}, 2^{(+)}, 4^{(*)}$ & 1381.90 & 183.45 & $1^{(+)}, 2^{(+)}, 3^{(*)}$ \\
&500  & 100  & 25  & 0.001  & 5537.00  & 565.70  & $2^{(-)}, 3^{(-)}, 4^{(-)}$ & 4457.91 & 512.51 & $1^{(+)}, 3^{(-)}, 4^{(-)}$ & 3162.58  & 392.02  & $1^{(+)}, 2^{(+)}, 4^{(-)}$ & 1561.59 & 210.73 & $1^{(+)}, 2^{(+)}, 3^{(+)}$ \\
&500  & 100  & 25  & 0.0001 & 9869.75  & 912.78  & $2^{(-)}, 3^{(-)}, 4^{(-)}$ & 4590.04 & 518.27 & $1^{(+)}, 3^{(+)}, 4^{(-)}$ & 7854.36  & 797.79  & $1^{(+)}, 2^{(-)}, 4^{(-)}$ & 1718.78 & 243.25 & $1^{(+)}, 2^{(+)}, 3^{(+)}$ \\
&500  & 100  & 50  & 0.01   & 4862.10  & 512.38  & $2^{(*)}, 3^{(-)}, 4^{(-)}$ & 4916.24 & 540.29 & $1^{(*)}, 3^{(-)}, 4^{(-)}$ & 2266.92  & 281.86  & $1^{(+)}, 2^{(+)}, 4^{(*)}$ & 2062.41 & 291.98 & $1^{(+)}, 2^{(+)}, 3^{(*)}$ \\
&500  & 100  & 50  & 0.001  & 7684.09  & 742.16  & $2^{(-)}, 3^{(-)}, 4^{(-)}$ & 5252.78 & 593.48 & $1^{(+)}, 3^{(*)}, 4^{(-)}$ & 5442.18  & 616.59  & $1^{(+)}, 2^{(*)}, 4^{(-)}$ & 2414.63 & 346.82 & $1^{(+)}, 2^{(+)}, 3^{(+)}$ \\
&500  & 100  & 50  & 0.0001 & 14435.49 & 1539.90 & $2^{(-)}, 3^{(*)}, 4^{(-)}$ & 5559.46 & 621.60 & $1^{(+)}, 3^{(+)}, 4^{(-)}$ & 13477.00 & 1291.40 & $1^{(*)}, 2^{(-)}, 4^{(-)}$ & 2730.43 & 400.21 & $1^{(+)}, 2^{(+)}, 3^{(+)}$ \\
&500  & 1000 & 25  & 0.01   & 2498.45  & 122.34  & $2^{(*)}, 3^{(-)}, 4^{(-)}$ & 2425.96 & 102.79 & $1^{(*)}, 3^{(-)}, 4^{(-)}$ & 1004.26  & 48.61   & $1^{(+)}, 2^{(+)}, 4^{(*)}$ & 896.22  & 77.77  & $1^{(+)}, 2^{(+)}, 3^{(*)}$ \\
&500  & 1000 & 25  & 0.001  & 4240.41  & 286.22  & $2^{(-)}, 3^{(-)}, 4^{(-)}$ & 2655.49 & 108.46 & $1^{(+)}, 3^{(+)}, 4^{(-)}$ & 2900.27  & 130.49  & $1^{(+)}, 2^{(-)}, 4^{(-)}$ & 1125.94 & 102.26 & $1^{(+)}, 2^{(+)}, 3^{(+)}$ \\
&500  & 1000 & 25  & 0.0001 & 8477.51  & 1091.61 & $2^{(-)}, 3^{(*)}, 4^{(-)}$ & 2837.93 & 119.76 & $1^{(+)}, 3^{(+)}, 4^{(-)}$ & 7526.05  & 817.63  & $1^{(*)}, 2^{(-)}, 4^{(-)}$ & 1331.74 & 126.31 & $1^{(+)}, 2^{(+)}, 3^{(+)}$ \\
&500  & 1000 & 50  & 0.01   & 3342.10  & 178.28  & $2^{(*)}, 3^{(-)}, 4^{(-)}$ & 3195.27 & 148.83 & $1^{(*)}, 3^{(-)}, 4^{(-)}$ & 1904.68  & 83.35   & $1^{(+)}, 2^{(+)}, 4^{(-)}$ & 1719.30 & 130.07 & $1^{(+)}, 2^{(+)}, 3^{(+)}$ \\
&500  & 1000 & 50  & 0.001  & 6440.89  & 611.99  & $2^{(-)}, 3^{(-)}, 4^{(-)}$ & 3633.98 & 165.28 & $1^{(+)}, 3^{(+)}, 4^{(-)}$ & 5282.87  & 380.55  & $1^{(+)}, 2^{(-)}, 4^{(-)}$ & 2162.80 & 167.66 & $1^{(+)}, 2^{(+)}, 3^{(+)}$ \\
&500  & 1000 & 50  & 0.0001 & 11975.96 & 2404.11 & $2^{(-)}, 3^{(*)}, 4^{(-)}$ & 3983.10 & 193.16 & $1^{(+)}, 3^{(+)}, 4^{(-)}$ & 11652.85 & 2129.08 & $1^{(*)}, 2^{(-)}, 4^{(-)}$ & 2555.44 & 201.20 & $1^{(+)}, 2^{(+)}, 3^{(+)}$ \\
\cmidrule{2-17}
\parbox[t]{2mm}{\multirow{12}{*}{\rotatebox[origin=c]{90}{bounded-strongly-correlated}}} 
&500  & 100  & 25  & 0.01   & 3287.08  & 390.63  & $2^{(*)}, 3^{(-)}, 4^{(-)}$ & 3333.50 & 389.40 & $1^{(*)}, 3^{(-)}, 4^{(-)}$ & 1523.05  & 166.13  & $1^{(+)}, 2^{(+)}, 4^{(*)}$ & 1400.05 & 125.34 & $1^{(+)}, 2^{(+)}, 3^{(*)}$ \\
&500  & 100  & 25  & 0.001  & 4763.94  & 780.09  & $2^{(-)}, 3^{(-)}, 4^{(-)}$ & 3509.95 & 428.48 & $1^{(+)}, 3^{(*)}, 4^{(-)}$ & 3251.30  & 454.50  & $1^{(+)}, 2^{(*)}, 4^{(-)}$ & 1583.87 & 142.56 & $1^{(+)}, 2^{(+)}, 3^{(+)}$ \\
&500  & 100  & 25  & 0.0001 & 9387.44  & 2060.47 & $2^{(-)}, 3^{(*)}, 4^{(-)}$ & 3674.83 & 446.95 & $1^{(+)}, 3^{(+)}, 4^{(-)}$ & 7843.42  & 1450.53 & $1^{(*)}, 2^{(-)}, 4^{(-)}$ & 1745.45 & 156.64 & $1^{(+)}, 2^{(+)}, 3^{(+)}$ \\
&500  & 100  & 50  & 0.01   & 3998.90  & 528.67  & $2^{(*)}, 3^{(-)}, 4^{(-)}$ & 4052.63 & 516.41 & $1^{(*)}, 3^{(-)}, 4^{(-)}$ & 2327.44  & 290.44  & $1^{(+)}, 2^{(+)}, 4^{(*)}$ & 2107.63 & 210.29 & $1^{(+)}, 2^{(+)}, 3^{(*)}$ \\
&500  & 100  & 50  & 0.001  & 7092.19  & 1348.32 & $2^{(-)}, 3^{(-)}, 4^{(-)}$ & 4405.88 & 575.40 & $1^{(+)}, 3^{(+)}, 4^{(-)}$ & 5516.23  & 906.14  & $1^{(+)}, 2^{(-)}, 4^{(-)}$ & 2465.45 & 254.71 & $1^{(+)}, 2^{(+)}, 3^{(+)}$ \\
&500  & 100  & 50  & 0.0001 & 13743.81 & 3964.29 & $2^{(-)}, 3^{(*)}, 4^{(-)}$ & 4736.65 & 625.25 & $1^{(+)}, 3^{(+)}, 4^{(-)}$ & 12936.38 & 3124.07 & $1^{(*)}, 2^{(-)}, 4^{(-)}$ & 2790.44 & 288.47 & $1^{(+)}, 2^{(+)}, 3^{(+)}$ \\
&500  & 1000 & 25  & 0.01   & 1971.76  & 244.69  & $2^{(*)}, 3^{(-)}, 4^{(-)}$ & 1892.49 & 231.54 & $1^{(*)}, 3^{(-)}, 4^{(-)}$ & 823.88   & 116.26  & $1^{(+)}, 2^{(+)}, 4^{(*)}$ & 730.56  & 60.10  & $1^{(+)}, 2^{(+)}, 3^{(*)}$ \\
&500  & 1000 & 25  & 0.001  & 3455.36  & 486.91  & $2^{(-)}, 3^{(-)}, 4^{(-)}$ & 2063.78 & 247.58 & $1^{(+)}, 3^{(*)}, 4^{(-)}$ & 2265.75  & 288.37  & $1^{(+)}, 2^{(*)}, 4^{(-)}$ & 905.61  & 74.11  & $1^{(+)}, 2^{(+)}, 3^{(+)}$ \\
&500  & 1000 & 25  & 0.0001 & 6931.91  & 1328.75 & $2^{(-)}, 3^{(*)}, 4^{(-)}$ & 2226.64 & 246.62 & $1^{(+)}, 3^{(+)}, 4^{(-)}$ & 5709.88  & 949.12  & $1^{(*)}, 2^{(-)}, 4^{(-)}$ & 1059.64 & 79.37  & $1^{(+)}, 2^{(+)}, 3^{(+)}$ \\
&500  & 1000 & 50  & 0.01   & 2694.66  & 351.22  & $2^{(*)}, 3^{(-)}, 4^{(-)}$ & 2539.21 & 302.30 & $1^{(*)}, 3^{(-)}, 4^{(-)}$ & 1513.94  & 197.83  & $1^{(+)}, 2^{(+)}, 4^{(*)}$ & 1358.32 & 120.78 & $1^{(+)}, 2^{(+)}, 3^{(*)}$ \\
&500  & 1000 & 50  & 0.001  & 5297.71  & 876.15  & $2^{(-)}, 3^{(-)}, 4^{(-)}$ & 2883.68 & 335.21 & $1^{(+)}, 3^{(+)}, 4^{(-)}$ & 4044.39  & 603.89  & $1^{(+)}, 2^{(-)}, 4^{(-)}$ & 1694.02 & 145.32 & $1^{(+)}, 2^{(+)}, 3^{(+)}$ \\
&500  & 1000 & 50  & 0.0001 & 9534.51  & 2279.76 & $2^{(-)}, 3^{(*)}, 4^{(-)}$ & 3182.34 & 366.82 & $1^{(+)}, 3^{(+)}, 4^{(-)}$ & 8831.99  & 1787.84 & $1^{(*)}, 2^{(-)}, 4^{(-)}$ & 1990.19 & 163.30 & $1^{(+)}, 2^{(+)}, 3^{(+)}$ \\
\bottomrule
\end{tabular}
\end{adjustbox}
\end{table*}

\begin{table*} 
\caption{Statistical results of total offline error for (1+1)-EA and POSDC with large change in the dynamic constraint}
\label{table.2000}
\begin{adjustbox}{max width=2\columnwidth}
\centering
\begin{tabular}{@{}cllllrrlrrlrrlrrl@{}}
\toprule
&$r$ &$\tau$ &$\delta$ &$\alpha$ &\multicolumn{3}{c}{(1+1)-EA-Chebyshev (1)} & \multicolumn{3}{c}{(1+1)-EA-Chernoff (2)} & \multicolumn{3}{c}{POSDC-Chebyshev (3)} & \multicolumn{3}{c}{POSDC-Chernoff (4)} \\
\cmidrule{6-8} \cmidrule{9-11} \cmidrule{12-14} \cmidrule{15-17}
\parbox[t]{2mm}{\multirow{16}{*}{\rotatebox[origin=c]{90}{uncorrelated}}} &     &      &       &          & Mean      & Std     & Stat    & Mean     & Std    & Stat     & Mean      & Std     & Stat    & Mean     & Std    & Stat    \\
\midrule
&2000 & 100  & 25  & 0.01   & 5948.81  & 569.75  & $2^{(*)}, 3^{(-)}, 4^{(-)}$ & 6018.58 & 560.15 & $1^{(*)}, 3^{(-)}, 4^{(-)}$ & 1931.58  & 366.87  & $1^{(+)}, 2^{(+)}, 4^{(*)}$ & 1909.30 & 381.82 & $1^{(+)}, 2^{(+)}, 3^{(*)}$ \\
&2000 & 100  & 25  & 0.001  & 6387.66  & 508.07  & $2^{(*)}, 3^{(-)}, 4^{(-)}$ & 6074.30 & 577.59 & $1^{(*)}, 3^{(-)}, 4^{(-)}$ & 3133.97  & 466.62  & $1^{(+)}, 2^{(+)}, 4^{(-)}$ & 2009.58 & 388.93 & $1^{(+)}, 2^{(+)}, 3^{(+)}$ \\
&2000 & 100  & 25  & 0.0001 & 10237.76 & 594.11  & $2^{(-)}, 3^{(-)}, 4^{(-)}$ & 6170.77 & 579.85 & $1^{(+)}, 3^{(*)}, 4^{(-)}$ & 7010.05  & 698.68  & $1^{(+)}, 2^{(*)}, 4^{(-)}$ & 2102.31 & 402.86 & $1^{(+)}, 2^{(+)}, 3^{(+)}$ \\
&2000 & 100  & 50  & 0.01   & 6328.16  & 563.72  & $2^{(*)}, 3^{(-)}, 4^{(-)}$ & 6399.28 & 594.11 & $1^{(*)}, 3^{(-)}, 4^{(-)}$ & 2476.40  & 410.13  & $1^{(+)}, 2^{(+)}, 4^{(*)}$ & 2378.09 & 430.97 & $1^{(+)}, 2^{(+)}, 3^{(*)}$ \\
&2000 & 100  & 50  & 0.001  & 8198.35  & 556.15  & $2^{(-)}, 3^{(-)}, 4^{(-)}$ & 6592.78 & 582.25 & $1^{(+)}, 3^{(-)}, 4^{(-)}$ & 4963.66  & 603.48  & $1^{(+)}, 2^{(+)}, 4^{(-)}$ & 2601.97 & 454.21 & $1^{(+)}, 2^{(+)}, 3^{(+)}$ \\
&2000 & 100  & 50  & 0.0001 & 15154.74 & 668.43  & $2^{(-)}, 3^{(-)}, 4^{(-)}$ & 6794.26 & 590.57 & $1^{(+)}, 3^{(+)}, 4^{(-)}$ & 12102.77 & 742.79  & $1^{(+)}, 2^{(-)}, 4^{(-)}$ & 2806.32 & 467.27 & $1^{(+)}, 2^{(+)}, 3^{(+)}$ \\
&2000 & 1000 & 25  & 0.01   & 3027.71  & 377.45  & $2^{(*)}, 3^{(-)}, 4^{(-)}$ & 2966.82 & 374.67 & $1^{(*)}, 3^{(-)}, 4^{(-)}$ & 974.46   & 188.68  & $1^{(+)}, 2^{(+)}, 4^{(*)}$ & 874.60  & 190.49 & $1^{(+)}, 2^{(+)}, 3^{(*)}$ \\
&2000 & 1000 & 25  & 0.001  & 4429.57  & 502.95  & $2^{(-)}, 3^{(-)}, 4^{(-)}$ & 3120.19 & 381.19 & $1^{(+)}, 3^{(*)}, 4^{(-)}$ & 2556.49  & 361.79  & $1^{(+)}, 2^{(*)}, 4^{(-)}$ & 1040.90 & 210.60 & $1^{(+)}, 2^{(+)}, 3^{(+)}$ \\
&2000 & 1000 & 25  & 0.0001 & 8650.49  & 843.44  & $2^{(-)}, 3^{(-)}, 4^{(-)}$ & 3255.02 & 416.14 & $1^{(+)}, 3^{(+)}, 4^{(-)}$ & 6959.35  & 732.40  & $1^{(+)}, 2^{(-)}, 4^{(-)}$ & 1186.47 & 235.49 & $1^{(+)}, 2^{(+)}, 3^{(+)}$ \\
&2000 & 1000 & 50  & 0.01   & 3714.96  & 442.43  & $2^{(*)}, 3^{(-)}, 4^{(-)}$ & 3558.83 & 432.43 & $1^{(*)}, 3^{(-)}, 4^{(-)}$ & 1704.27  & 270.78  & $1^{(+)}, 2^{(+)}, 4^{(*)}$ & 1513.22 & 278.61 & $1^{(+)}, 2^{(+)}, 3^{(*)}$ \\
&2000 & 1000 & 50  & 0.001  & 6464.63  & 664.72  & $2^{(-)}, 3^{(-)}, 4^{(-)}$ & 3872.78 & 476.96 & $1^{(+)}, 3^{(*)}, 4^{(-)}$ & 4697.95  & 565.23  & $1^{(+)}, 2^{(*)}, 4^{(-)}$ & 1845.29 & 322.04 & $1^{(+)}, 2^{(+)}, 3^{(+)}$ \\
&2000 & 1000 & 50  & 0.0001 & 13380.75 & 1349.08 & $2^{(-)}, 3^{(*)}, 4^{(-)}$ & 4159.55 & 509.55 & $1^{(+)}, 3^{(+)}, 4^{(-)}$ & 12017.35 & 1151.98 & $1^{(*)}, 2^{(-)}, 4^{(-)}$ & 2138.14 & 367.74 & $1^{(+)}, 2^{(+)}, 3^{(+)}$ \\
\cmidrule{2-17}
\parbox[t]{2mm}{\multirow{12}{*}{\rotatebox[origin=c]{90}{bounded-strongly-correlated}}}
&2000 & 100  & 25  & 0.01   & 4560.36  & 185.97  & $2^{(*)}, 3^{(-)}, 4^{(-)}$ & 4568.83 & 197.34 & $1^{(*)}, 3^{(-)}, 4^{(-)}$ & 1840.56  & 84.26   & $1^{(+)}, 2^{(+)}, 4^{(*)}$ & 1712.51 & 123.38 & $1^{(+)}, 2^{(+)}, 3^{(*)}$ \\
&2000 & 100  & 25  & 0.001  & 5784.27  & 319.36  & $2^{(-)}, 3^{(-)}, 4^{(-)}$ & 4718.52 & 189.73 & $1^{(+)}, 3^{(-)}, 4^{(-)}$ & 3795.55  & 168.68  & $1^{(+)}, 2^{(+)}, 4^{(-)}$ & 1896.19 & 95.79  & $1^{(+)}, 2^{(+)}, 3^{(+)}$ \\
&2000 & 100  & 25  & 0.0001 & 12130.92 & 1256.94 & $2^{(-)}, 3^{(-)}, 4^{(-)}$ & 4879.27 & 180.63 & $1^{(+)}, 3^{(+)}, 4^{(-)}$ & 9177.87  & 928.21  & $1^{(+)}, 2^{(-)}, 4^{(-)}$ & 2063.82 & 79.10  & $1^{(+)}, 2^{(+)}, 3^{(+)}$ \\
&2000 & 100  & 50  & 0.01   & 5337.16  & 166.82  & $2^{(*)}, 3^{(-)}, 4^{(-)}$ & 5291.89 & 176.82 & $1^{(*)}, 3^{(-)}, 4^{(-)}$ & 2745.54  & 52.90   & $1^{(+)}, 2^{(+)}, 4^{(-)}$ & 2484.25 & 54.08  & $1^{(+)}, 2^{(+)}, 3^{(+)}$ \\
&2000 & 100  & 50  & 0.001  & 8834.24  & 746.73  & $2^{(-)}, 3^{(-)}, 4^{(-)}$ & 5653.89 & 184.06 & $1^{(+)}, 3^{(+)}, 4^{(-)}$ & 6452.92  & 512.65  & $1^{(+)}, 2^{(-)}, 4^{(-)}$ & 2862.14 & 42.48  & $1^{(+)}, 2^{(+)}, 3^{(+)}$ \\
&2000 & 100  & 50  & 0.0001 & 19641.09 & 2649.89 & $2^{(-)}, 3^{(-)}, 4^{(-)}$ & 5987.13 & 204.54 & $1^{(+)}, 3^{(+)}, 4^{(-)}$ & 15189.24 & 2029.02 & $1^{(+)}, 2^{(-)}, 4^{(-)}$ & 3205.05 & 58.24  & $1^{(+)}, 2^{(+)}, 3^{(+)}$ \\
&2000 & 1000 & 25  & 0.01   & 2508.30  & 264.64  & $2^{(*)}, 3^{(-)}, 4^{(-)}$ & 2390.56 & 233.31 & $1^{(*)}, 3^{(-)}, 4^{(-)}$ & 963.19   & 99.04   & $1^{(+)}, 2^{(+)}, 4^{(-)}$ & 828.44  & 45.06  & $1^{(+)}, 2^{(+)}, 3^{(+)}$ \\
&2000 & 1000 & 25  & 0.001  & 4120.10  & 557.29  & $2^{(-)}, 3^{(-)}, 4^{(-)}$ & 2581.26 & 264.54 & $1^{(+)}, 3^{(*)}, 4^{(-)}$ & 2568.42  & 347.75  & $1^{(+)}, 2^{(*)}, 4^{(-)}$ & 1010.08 & 58.55  & $1^{(+)}, 2^{(+)}, 3^{(+)}$ \\
&2000 & 1000 & 25  & 0.0001 & 8550.48  & 1602.26 & $2^{(-)}, 3^{(*)}, 4^{(-)}$ & 2745.14 & 281.97 & $1^{(+)}, 3^{(+)}, 4^{(-)}$ & 6518.39  & 1139.10 & $1^{(*)}, 2^{(-)}, 4^{(-)}$ & 1163.33 & 72.29  & $1^{(+)}, 2^{(+)}, 3^{(+)}$ \\
&2000 & 1000 & 50  & 0.01   & 3302.84  & 395.53  & $2^{(*)}, 3^{(-)}, 4^{(-)}$ & 3103.47 & 336.52 & $1^{(*)}, 3^{(-)}, 4^{(-)}$ & 1732.35  & 211.93  & $1^{(+)}, 2^{(+)}, 4^{(*)}$ & 1495.03 & 118.31 & $1^{(+)}, 2^{(+)}, 3^{(*)}$ \\
&2000 & 1000 & 50  & 0.001  & 6334.26  & 1019.02 & $2^{(-)}, 3^{(-)}, 4^{(-)}$ & 3455.32 & 379.88 & $1^{(+)}, 3^{(+)}, 4^{(-)}$ & 4582.94  & 717.91  & $1^{(+)}, 2^{(-)}, 4^{(-)}$ & 1844.94 & 154.35 & $1^{(+)}, 2^{(+)}, 3^{(+)}$ \\
&2000 & 1000 & 50  & 0.0001 & 12880.38 & 2979.96 & $2^{(-)}, 3^{(*)}, 4^{(-)}$ & 3757.97 & 409.95 & $1^{(+)}, 3^{(+)}, 4^{(-)}$ & 10047.45 & 2229.63 & $1^{(*)}, 2^{(-)}, 4^{(-)}$ & 2154.25 & 185.00 & $1^{(+)}, 2^{(+)}, 3^{(+)}$ \\
\bottomrule
\end{tabular}
\end{adjustbox}
\end{table*}

\begin{table*}
\caption{Statistical results of total offline error for NSGA-II with large change in the dynamic constraint ($r=2000$)}\label{table.nsga}
\begin{adjustbox}{max width=2\columnwidth}
\centering
\begin{tabular}{@{}cllrrlrrclrrlrrl@{}}
\toprule
$\tau$ &$\delta$ &$\alpha$ &\multicolumn{6}{c}{uncorrelated} && \multicolumn{6}{c}{bounded-strongly correlated}\\
&         &         &\multicolumn{3}{c}{NSGA-II-Chebyshev (5)} & \multicolumn{3}{c}{NSGA-II-Chernoff (6)} && \multicolumn{3}{c}{NSGA-II-Chebyshev (5)} & \multicolumn{3}{c}{NSGA-II-Chernoff (6)} \\
\cmidrule{4-6} \cmidrule{7-9} \cmidrule{11-13} \cmidrule{14-16}
&       &          & Mean      & Std     & Stat    & Mean     & Std    & Stat &    & Mean      & Std     & Stat    & Mean     & Std    & Stat    \\
\midrule
100  & 25  & 0.01   & 2215.77  & 295.97  & $3^{(*)}, 4^{(*)}, 6^{(*)}$ & 2130.59 & 279.40  & $3^{(*)}, 4^{(-)}, 5^{(*)}$ & &2390.79 & 189.51  & $3^{(-)}, 4^{(-)}, 6^{(*)}$ & 2234.28 & 194.74 & $3^{(-)}, 4^{(-)}, 5^{(*)}$ \\
100  & 25  & 0.001  & 3509.35  & 421.03  & $3^{(*)}, 4^{(-)}, 6^{(-)}$ & 2268.36 & 289.77 & $3^{(*)}, 4^{(*)}, 5^{(+)}$ & &4399.93 & 374.16  & $3^{(*)}, 4^{(-)}, 6^{(-)}$ & 2416.83 & 148.89 & $3^{(*)}, 4^{(*)}, 5^{(+)}$ \\
100  & 25  & 0.0001 & 7401.77  & 621.03  & $3^{(*)}, 4^{(-)}, 6^{(-)}$ & 2387.85 & 290.26 & $3^{(*)}, 4^{(*)}, 5^{(+)}$ & &9648.37 & 1205.88 & $3^{(*)}, 4^{(-)}, 6^{(-)}$ & 2652.75 & 166.96 & $3^{(+)}, 4^{(*)}, 5^{(+)}$ \\
100  & 50  & 0.01   & 2828.78  & 329.66  & $3^{(-)}, 4^{(-)}, 6^{(*)}$ & 2637.66 & 327.82 & $3^{(+)}, 4^{(*)}, 5^{(+)}$ & &3342.50  & 234.54  & $3^{(-)}, 4^{(-)}, 6^{(*)}$ & 3042.16 & 200.33 & $3^{(*)}, 4^{(-)}, 5^{(*)}$ \\
100  & 50  & 0.001  & 5358.32  & 552.16  & $3^{(*)}, 4^{(-)}, 6^{(-)}$ & 2905.99 & 352.18 & $3^{(*)}, 4^{(*)}, 5^{(*)}$ & &6993.80  & 744.54  & $3^{(*)}, 4^{(-)}, 6^{(-)}$ & 3439.87 & 215.47 & $3^{(+)}, 4^{(*)}, 5^{(+)}$ \\
100  & 50  & 0.0001 & 12392.67 & 677.71  & $3^{(*)}, 4^{(-)}, 6^{(-)}$ & 3150.29 & 392.47 & $3^{(+)}, 4^{(*)}, 5^{(+)}$ & &15160.00   & 2418.27 & $3^{(*)}, 4^{(-)}, 6^{(-)}$ & 3798.48 & 239.52 & $3^{(+)}, 4^{(*)}, 5^{(+)}$ \\
1000 & 25  & 0.01   & 1275.93  & 157.45  & $3^{(-)}, 4^{(-)}, 6^{(*)}$ & 1170.70  & 157.26 & $3^{(*)}, 4^{(-)}, 5^{(*)}$ & &1123.35 & 150.54  & $3^{(-)}, 4^{(-)}, 6^{(*)}$ & 1009.96 & 115.48 & $3^{(*)}, 4^{(-)}, 5^{(*)}$ \\
1000 & 25  & 0.001  & 2844.91  & 318.58  & $3^{(*)}, 4^{(-)}, 6^{(-)}$ & 1333.13 & 168.23 & $3^{(+)}, 4^{(*)}, 5^{(+)}$ & &2726.67 & 392.87  & $3^{(*)}, 4^{(-)}, 6^{(-)}$ & 1198.56 & 122.26 & $3^{(+)}, 4^{(*)}, 5^{(+)}$ \\
1000 & 25  & 0.0001 & 7228.72  & 700.04  & $3^{(*)}, 4^{(-)}, 6^{(-)}$ & 1495.80  & 180.51 & $3^{(+)}, 4^{(*)}, 5^{(+)}$ & &6597.74 & 1210.95 & $3^{(*)}, 4^{(-)}, 6^{(-)}$ & 1360.20  & 150.56 & $3^{(+)}, 4^{(*)}, 5^{(+)}$ \\
1000 & 50  & 0.01   & 2016.38  & 233.56  & $3^{(-)}, 4^{(-)}, 6^{(*)}$ & 1812.49 & 222.10  & $3^{(*)}, 4^{(*)}, 5^{(*)}$ & &1872.84 & 237.35  & $3^{(*)}, 4^{(-)}, 6^{(*)}$ & 1682.10  & 184.08 & $3^{(*)}, 4^{(*)}, 5^{(*)}$ \\
1000 & 50  & 0.001  & 4967.78  & 537.00     & $3^{(*)}, 4^{(-)}, 6^{(-)}$ & 2153.47 & 265.98 & $3^{(+)}, 4^{(*)}, 5^{(+)}$ & &4679.54 & 753.61  & $3^{(*)}, 4^{(-)}, 6^{(-)}$ & 2032.80  & 215.07 & $3^{(+)}, 4^{(*)}, 5^{(+)}$ \\
1000 & 50  & 0.0001 & 12192.94 & 1174.98 & $3^{(*)}, 4^{(-)}, 6^{(-)}$ & 2447.42 & 306.57 & $3^{(+)}, 4^{(*)}, 5^{(+)}$ & &9885.44 & 2350.72 & $3^{(*)}, 4^{(-)}, 6^{(-)}$ & 2342.47 & 253.88 & $3^{(+)}, 4^{(*)}, 5^{(+)}$ \\
\bottomrule
\end{tabular}
\end{adjustbox}
\end{table*}
We combine the parameters of $r$, $\tau$, $\alpha$ and, $\delta$ to produce DCCKP test problem instances for uncorrelated and bounded strongly correlated with different types of complexities. For instance, a test problem with $r=2000$, $\tau=100$, $\alpha=0.0001$ and $\delta=50$ represents the most difficult test problem; because the magnitude of dynamic change in the knapsack capacity is large and the capacity changes very fast every 100 iterations. Also, the allowable probability of chance-constraint violation is very tight, and the uncertainty interval in the weight of items is big.
 
We apply POSDC and (1+1)-EA integrated with Chebyshev and Chernoff inequality tails to DCCKP instances. Specifically, we investigate the following algorithms:
	\begin{itemize}
		\item (1+1)-EA with Chebyshev's inequality: (1)
		\item (1+1)-EA with Chernoff bound: (2)		
		\item POSDC with Chebyshev's inequality: (3)
		\item POSDC with Chernoff bound: (4) 
	\end{itemize} 
	Each algorithm initially runs for $10^4$ warm-up iterations before the first change in the capacity occurs and continues for $10^6$ iterations. Tables \ref{table.500} and \ref{table.2000} report the performance of single-objective and bi-objective optimization by the average and standard deviation of total offline error for 30 independent runs. Lower total offline error is better because it shows the algorithm was closer to the $P(x^*)$ for each iteration. Note that when the problem becomes more uncertain, the feasible region (without violating the probabilistic constraint) becomes more restrictive and the offline error will be increased.
 
Statistical comparisons are carried out by using the Kruskal-Wallis test with 95\% confidence interval integrated with the posteriori Bonferroni test to compare multiple solutions \cite{Corder2014}. The stat column shows the rank of each algorithm in the instances; If two algorithms can be compared with each other significantly, $X^{(+)}$ denotes that the current algorithm is outperforming algorithm $X$. Likewise, $X^{(-)}$ signifies the current algorithm is worse than the algorithm $X$ significantly. Otherwise, $X^{(*)}$ shows that the current algorithm is not different significantly with algorithm $X$. For example, numbers $1^{(+)}, 3^{(*)}, 4^{(-)}$ denote the pairwise performance of algorithm (2). The numbers show that algorithm (2) is statistically better than algorithm (1); it is not different from algorithm (3) and it is inferior to algorithm (4).

Table \ref{table.500} lists the results when $r$ is 500. We observe that when the environment of the problem becomes more complex, finding a solution which has a close distance to the optimal solution is harder. As $\tau$ decreases, $\delta$ increases and $\alpha$ becomes tighter, the offline error for both (1+1)-EA and POSDC increases. However, as the problem becomes more difficult to solve, POSDC obtains solutions with a lower total offline error and lower standard deviation in comparison with (1+1)-EA. We also find that the algorithms which use Chernoff bound outperform other algorithms which use the Chebyshev's inequality. 
 
Table \ref{table.2000} lists our results when $r$ is 2000. We observe that POSDC can obtain better solutions in comparison with (1+1)-EA. When we consider a bigger magnitude of changes in the constraint bound, the population size of non-dominated solutions in POSDC is bigger than when $r$ is 500; because $\eta$ is equal to $r$, POSDC covers a bigger range of solutions which leads to a bigger population.
 
Therefore, when the changes occur faster (smaller $\tau$), POSDC has less time to evolve its population. POSDC only mutates one individual chosen randomly in its population, leading to a lower chance of choosing the best individual for the mutation in its population. In contrast, (1+1)-EA only handles one individual, mutates and improves it on all iterations. Introducing our second objective function for the bi-objective optimization approach helps POSDC to tackle all these drawbacks and outperform its counterpart single-objective approach; because trade-off solutions contain more information in principle of finding better solutions.

For further investigation of our bi-objective optimization, we also apply the Non-dominated Sorting Genetic Algorithm (NSGA-II)~\cite{NSGA2}, which is a state of the art multi-objective EA when dealing with two objectives. We run NSGA-II with a population size of $20$ using Chebyshev and Chernoff inequality tails which are algorithms (5) and (6), respectively in Table \ref{table.nsga}. Table \ref{table.nsga} shows the results of NSGA-II when $r$ is 2000 for uncorrelated and bounded strongly correlated instances and compares the performance of NSGA-II with POSDC. For brevity, we only report stats of comparison between NSGA-II and POSDC.

To have a fair comparison, we modify NSGA-II to keep the best-obtained solution for the given knapsack bound $C$ in each iteration. Table \ref{table.nsga} shows that in most of the instances, NSGA-II performs as good as POSDC when using the Chernoff bound. However, POSDC can outperform NSGA-II in instances where $\delta=25$ and $\alpha=0.01$, which is the most straightforward instance. The main difference between NSGA-II and POSDC is the selection mechanism. NSGA-II uses the crowding distance sorting to maintain diversity through the evolution of its population. This comparison can point out the possible research line to further investigate state-of-art non-baseline EAs and multi-objective EAs solving DCCKPs. 

\section{\uppercase{Conclusions}} \label{sec.conclusion}
In this paper, we dealt with the dynamic chance-constrained knapsack problem where the constraint bound changes dynamically over time, and item weights are uncertain. The key part of our approach is to tackle the dynamic and stochastic components of an optimization problem in a holistic approach. For this purpose and to apply bi-objective optimization to the problem, we developed an objective $C^*$ which calculates for a given solution $x$ the smallest possible bound for which $x$ would meet the chance constraint. This objective function allows keeping a set of non-dominated solutions with different $C^*$ where an appropriate solution can be used to track the optimum after the dynamic constraint bound has changed. As it is hard to calculate the bound $C^*(x)$ in the stochastic setting exactly, we have shown how to calculate upper bounds for $C^*(x)$ based on Chernoff bound and Chebyshev's inequality. We evaluated the bi-objective optimization for a wide range of chance-constrained knapsack problems with dynamically changing constraint bounds. The results show that the bi-objective optimization with the introduced additional objective function can obtain better results than single-objective optimization in most cases. Note that we also applied NSGA-II to the problem to point out possible improvements by using state of the art algorithms. It would be interesting for future work to extend these investigations. In addition, our approach is not limited to dynamic chance-constrained knapsack problems and the formulation can be adapted to a wide range of other problems where we would formulate a similar second objective to deal with the chance constraint.
\ignore{
 We introduced a new second objective function to solve a single-objective problem. To investigate our approach, we applied a Pareto optimization method, namely, POSDC to the knapsack problem test problems and compared it with its single-objective counterpart (1+1)-EA. Our results demonstrate that our approach leads to obtaining more efficient solutions in comparison with the single-objective approach. Moreover, our approach using its population can adapt to the dynamic environment of our problem. POSDC can track the changes, and be highly effective even when the environment of the problem is highly noisy. (1+1)-EA only works with one individual and mutates it to produce a better solution in each step. However, our Pareto optimization approach employs a multi-individual population to solve the problem. Using a population brings an additional overhead in terms of the population size. However, the experiments show that our approach benefits with computing a diverse set of potential solutions which caters for dynamic constraint changes. Introducing the additional objective function to the problem helps the Pareto optimization algorithm to find trade-off solutions with respect to the profit and possible constraint bounds leading to obtaining better results in comparison with (1+1)-EA. 
}
\ack 
This work has been supported by the Australian Research Council through grant DP160102401 and by the South Australian Government through the Research Consortium "Unlocking Complex Resources through Lean Processing".
\bibliography{bib}

\begin{thebibliography}{10}

\bibitem{Ben2009}
Aharon Ben-Tal, Laurent~El Ghaoui, and Arkadi Nemirovski, {\em Robust
  Optimization}, Princeton Series in Applied Mathematics, Princeton University
  Press, 2009.

\bibitem{Casella2002}
George Casella and Roger~L Berger, {\em Statistical inference}, volume~2,
  Duxbury Press, 2002.

\bibitem{Charnes1959}
Abraham Charnes and William~W Cooper, `Chance-constrained programming', {\em
  Management science}, {\bf 6}(1),  73--79, (1959).

\bibitem{Chiong2012}
Raymond Chiong, Thomas Weise, and Zbigniew Michalewicz, {\em Variants of
  evolutionary algorithms for real-world applications}, Springer, 2012.

\bibitem{Corder2014}
Gregory~W Corder and Dale~I Foreman, {\em Nonparametric statistics: A
  step-by-step approach}, John Wiley \& Sons, 2014.

\bibitem{DasGupta2008}
Anirban DasGupta, {\em A Collection of Inequalities in Probability, Linear
  Algebra, and Analysis},  633--687, Springer New York, New York, NY, 2008.

\bibitem{deb2001a}
Kalyanmoy Deb, {\em {Multi-objective optimization using evolutionary
  algorithms}}, Wiley, 2001.

\bibitem{NSGA2}
Kalyanmoy Deb, Amrit Pratap, Sameer Agarwal, and T.~Meyarivan, `A fast and
  elitist multiobjective genetic algorithm: {NSGA}-{II}', {\em IEEE
  Transactions on Evolutionary Computation}, {\bf 6}(2),  182--197, (April
  2002).

\bibitem{DoerrAAAI20}
Benjamin Doerr, Carola Doerr, Aneta Neumann, Frank Neumann, and Andrew~M.
  Sutton, `Optimization of chance-constrained submodular functions', in {\em
  Proc. of {AAAI}}, (2020).
\newblock to appear.

\bibitem{DROSTE2002}
Stefan Droste, Thomas Jansen, and Ingo Wegener, `On the analysis of the (1+1)
  evolutionary algorithm', {\em Theoretical Computer Science}, {\bf 276}(1),
  51 -- 81, (2002).

\bibitem{Farina2016}
Marcello Farina, Luca Giulioni, and Riccardo Scattolini, `Stochastic linear
  model predictive control with chance constraints--a review', {\em Journal of
  Process Control}, {\bf 44},  53--67, (2016).

\bibitem{Giel2003}
Oliver Giel, `Expected runtimes of a simple multi-objective evolutionary
  algorithm', in {\em Proc. of CEC}, volume~3, pp. 1918--1925. Citeseer,
  (2003).

\bibitem{Kellerer2004}
Hans Kellerer, Ulrich Pferschy, and David Pisinger, `Introduction to
  {NP}-completeness of knapsack problems', in {\em Knapsack problems},
  483--493, Springer, (2004).

\bibitem{Kleinberg2000}
Jon Kleinberg, Yuval Rabani, and {\'E}va Tardos, `Allocating bandwidth for
  bursty connections', {\em SIAM Journal on Computing}, {\bf 30}(1),  191--217,
  (2000).

\bibitem{DBLP:journals/orl/KlopfensteinN08}
Olivier Klopfenstein and Dritan Nace, `A robust approach to the
  chance-constrained knapsack problem', {\em Oper. Res. Lett.}, {\bf 36}(5),
  628--632, (2008).

\bibitem{Li2015}
Zhuangzhi Li and Zukui Li, `Chance constrained planning and scheduling under
  uncertainty using robust optimization approximation', {\em
  IFAC-PapersOnLine}, {\bf 48}(8),  1156--1161, (2015).

\bibitem{McDiarmid1998}
Colin McDiarmid, {\em Concentration},  195--248, Springer Berlin Heidelberg,
  1998.

\bibitem{Michalewicz1994}
Zbigniew Michalewicz and Jaros{\l}aw Arabas, `Genetic algorithms for the 0/1
  knapsack problem', in {\em International Symposium on Methodologies for
  Intelligent Systems}, pp. 134--143. Springer, (1994).

\bibitem{Motwani1995}
Rajeev Motwani and Prabhakar Raghavan, {\em Randomized algorithms}, Cambridge
  University Press, 1995.

\bibitem{Neumann2006}
Frank Neumann and Ingo Wegener, `Minimum spanning trees made easier via
  multi-objective optimization', {\em Natural Computing}, {\bf 5}(3),
  305--319, (2006).

\bibitem{Nguyen2012}
Trung~Thanh Nguyen, Shengxiang Yang, and Juergen Branke, `Evolutionary dynamic
  optimization: A survey of the state of the art', {\em Swarm and Evolutionary
  Computation}, {\bf 6},  1--24, (2012).

\bibitem{TTP2014}
Sergey Polyakovskiy, Mohammad~Reza Bonyadi, Markus Wagner, Zbigniew
  Michalewicz, and Frank Neumann, `A comprehensive benchmark set and heuristics
  for the traveling thief problem', in {\em Proceedings of the Genetic and
  Evolutionary Computation Conference, {GECCO} 2014}, pp. 477--484. ACM,
  (2014).

\bibitem{DBLP:conf/ijcai/QianSYT17}
Chao Qian, Jing{-}Cheng Shi, Yang Yu, and Ke~Tang, `On subset selection with
  general cost constraints', in {\em International Joint Conference on
  Artificial Intelligence, {IJCAI} 2017}, pp. 2613--2619, (2017).

\bibitem{DBLP:conf/nips/QianYZ15}
Chao Qian, Yang Yu, and Zhi{-}Hua Zhou, `Subset selection by {P}areto
  optimization', in {\em Advances in Neural Information Processing Systems 28:
  Annual Conference on Neural Information Processing Systems, {NIPS} 2015}, pp.
  1774--1782, (2015).

\bibitem{Roostapour2018}
Vahid Roostapour, Aneta Neumann, and Frank Neumann, `On the performance of
  baseline evolutionary algorithms on the dynamic knapsack problem', in {\em
  Parallel Problem Solving from Nature, {PPSN} {XV} 2018}, Lecture Notes in
  Computer Science, pp. 158--169. Springer, (2018).

\bibitem{Roostapour2019}
Vahid Roostapour, Aneta Neumann, Frank Neumann, and Tobias Friedrich, `Pareto
  optimization for subset selection with dynamic cost constraints', in {\em
  Proceedings of the AAAI Conference on Artificial Intelligence}, volume~33,
  pp. 2354--2361, (2019).

\bibitem{Yue2019}
Yue Xie, Oscar Harper, Hirad Assimi, Aneta Neumann, and Frank Neumann,
  `Evolutionary algorithms for the chance-constrained knapsack problem', in
  {\em Proceedings of the Genetic and Evolutionary Computation Conference,
  {GECCO} 2019}, pp. 338--346. ACM, (2019).

\end{thebibliography}
\end{document}